%% file: main.tex
\documentclass[dvipsnames,]{article} %
\usepackage{iclr2026_conference_unanon,times}

\usepackage{hyperref}
\usepackage{url}

\usepackage[utf8]{inputenc} %
\usepackage[T1]{fontenc}    %
\usepackage{booktabs}       %
\usepackage{amsfonts}       %
\usepackage{nicefrac}       %
\usepackage{microtype}      %
\usepackage{colortbl}

\usepackage[colorinlistoftodos,prependcaption,textsize=tiny]{todonotes}
\usepackage{bbm}
\usepackage{amsmath}
\usepackage{amsthm}
\usepackage{physics}
\usepackage{mathtools}
\usepackage{tikz-cd}

\usepackage{algorithm}
\usepackage{algorithmic}
\usepackage{subcaption}
\usepackage{epigraph}

\AtBeginDocument{}

\usepackage{caption}
\DeclareCaptionType[fileext=lod]{prompt}

\usepackage{rotating}

\newtheorem{theorem}{Theorem}
\newtheorem{lemma}{Lemma}
\newtheorem{remark}{Remark}
\newtheorem{corollary}{Corollary}

\newcommand{\secref}[1]{Section~\ref{sec:#1}}

\newcommand{\appref}[1]{Appendix~\ref{app:#1}}

\newcommand{\figref}[1]{Figure~\ref{fig:#1}}

\newcommand{\tabref}[1]{Table~\ref{tab:#1}}

\newcommand{\promptref}[1]{Prompt~\ref{pro:#1}}

\newcommand{\lemref}[1]{Lemma~\ref{lem:#1}}
\newcommand{\corref}[1]{Corollary~\ref{cor:#1}}

\DeclareRobustCommand{\DE}[3]{#3}
\DeclareRobustCommand{\VAN}[3]{#3}

\newcommand{\AoE}[0]{AoE II}

\title{If LLMs Have Human-Like Attributes, Then So Does \textit{Age of Empires II}}

\author{%
  Adrian de Wynter \\
  Microsoft \& The University of York \\
  \texttt{adewynter@microsoft.com}
}
\begin{document}
\maketitle
\begin{abstract}
\input{abstract}
\end{abstract}

\input{content}

\DeclareRobustCommand{\DE}[3]{#2}
\DeclareRobustCommand{\VAN}[3]{#2}
\bibliographystyle{abbrvnat}
\bibliography{biblio}

\newpage
\DeclareRobustCommand{\DE}[3]{#3}
\DeclareRobustCommand{\VAN}[3]{#3}

\appendix
\input{appendix}

\end{document}

%% file: abstract.tex
Much research has been carried out on large language models (LLMs) and LLM-powered agentic workflows. 
However, many works within the field state emergence of, ascribe to, or assume, generalised anthropomorphic attributes to them (e.g., morality or understanding of natural language). 
Our goal is not to argue in favour or against the existence of these attributes, but 
to point out that these conclusions could be incorrect. 
For this we build and train a simple neural network on the videogame \textit{Age of Empires II}, and note that any entity in a sufficiently-powerful substrate, such as LEGO or the Greater Boston Area, could also present such attributes. 
Hence, the purported anthropomorphic attributes of LLMs are empirically \textit{non-unique}: although some properties (e.g., responses to prompts) could remain invariant, others, such as the interpretation of their perceived behaviour, might change with the substrate. 
Thus, any empirically-grounded discussion on these attributes requires explicit measurement criteria; otherwise the interpretation is left to the representation. 
We then show that assuming that these attributes exist or not in a system, independent of the substrate and in a generalised way, leads to either circular or uninformative conclusions. 
This is regardless of the experimenter's viewpoint on the subject, or whether the outcome shows existence or non-existence. 
Finally we propose a `null' assumption, where one assumes LLM non-uniqueness instead of assuming anthropomorphic attributes to set up an experiment, along with examples of it. 
We also discuss potential objections to our work, briefly survey the field, and prove that \textit{Age of Empires II} is functionally- and Turing-complete. %

%% file: content.tex
\section{Introduction}

Even though large language models (LLMs) are comparatively new, they are both widespread and poorly understood. 
A big reason why this occurs is because the capabilities of LLMs,\footnote{We use `LLMs' as shorthand for `LLMs and LLM-powered agentic systems'.} coupled with their \textit{apparent} human attributes (e.g., communication skills), lead people to anthropomorphise them. 
This is reasonable: after all, although convincing dialogue systems have existed for over half a century (e.g., ELIZA; \citealt{eliza}), LLM-based chatbots incarnate a never-seen-before entity with abilities requiring an explanation starting from a familiar place. 

This has led to evaluations in various areas, such as (e.g.) theory of mind, learning and understanding, and psychology; all with variable results. 
What is common to some of these studies, however, is that they test and ascribe blanket human-like properties (e.g., anxiety or morality; \citealt{anthropicmind,NEURIPS2023_a2cf225b,zhou-etal-2024-rethinking}) to these LLMs \textit{while considering them the central subject of the experiment}. 
Regardless of these evaluations' results being positive or negative, their core assumption--that LLMs possess anthropomorphic attributes--influences the experiment's planning through (e.g.) the design of the test set, the interpretation of natural-language outputs, and even its null hypothesis. 
In turn, this directly impacts (and distorts; \citealt{anthroaihype}) the conclusions made.

In this paper we leverage these observations to show that, in LLM research, \textbf{assuming that general anthropomorphic properties exist or not as part of their measurement is fundamentally flawed}. 
For this, we begin by implementing and training a neural network in \textit{Age of Empires II} (\AoE{}).\footnote{A 1999 real-time strategy (RTS) game originally published by Ensemble Studios. 
All products, company names, brand names, trademarks, and images are properties of their respective owners and are only reproduced here with the educational purpose of illustrating mathematical theorems and philosophical discussion.} 
Although it might seem like a fun exercise, wholly unrelated to the topic of anthropomorphism in LLM research, we note that this immediately implies that (1) \textbf{any sufficiently powerful substrate could implement an entity equivalent to an LLM}; and (2) \textbf{said implementation alters the representation of an LLM}, and thus could affect its \textit{perceived properties}. 

Given that LLMs are sufficiently effective at mimicking anthropomorphic attributes to some extent, it follows that said mimicry--or true anthropomorphic behaviour, depending on the view--is not specific to LLMs as entities existing within a computer. Hence, LLMs are \textit{non-unique}--that is, implementations in other substrates could preserve some of their properties (e.g., prompt-output maps), but not their deanthropomorphic qualities. 
It then follows that the perception and interpretation of these will change. 
Thus, any discussion based on empirical observations requires explicit measurement criteria, with likewise explicit statements of which aspects should generalise across substrates. 

For this, suppose that the scientist takes an interpretive stance with respect to some framework (e.g., a computational theory of mind) wherein the attribute is taken to possibly exist in the system \textit{regardless} of substrate. 
We argue that this approach will draw unsound conclusions. 
For this we show that `accepting' such a framework in order to make claims--generalised or not--about anthropomorphic attributes leads to either a circular or an uninformative conclusion. 
Crucially, we also show that the same result holds if rejecting it. 
Thus, \textbf{assuming the existence or non-existence of generalised anthropomorphic attributes} in order to test a hypothesis \textbf{proving or disproving their existence is flawed}. %
It then follows that \textbf{conclusions from these experiments}, positive or negative, \textbf{do not support the claims}. 
Remark that the above is independent on the framework's validity or its acceptance/rejection; and, indeed, it is independent of the choice of framework.
Also note that such assumptions do not need to be made explicitly: for example, a paper disproving the ability of LLMs to factually `explain themselves' already assumes some level of self-awareness. 

The above also implies that \textbf{said attributes may be measured truthfully by approximation, if they do not claim generalisability}, and provided that \textit{these assumptions are not made}. 
We then propose a `null' assumption where LLM non-uniqueness is factored in by not making any statements on the existence or non-existence of anthropomorphic attributes in the system.

\subsection{Contributions}

Our goal is \textbf{not} to argue in favour or against the existence of anthropomorphic attributes in LLMs; the validity of any theory of mind; and decidedly not to discuss implications on consciousness or the mind-body problem as it relates to AI. 
This is because the first requires a well-defined measurement (the central theme of this paper); and for the latter there are no widely-accepted experimental protocols or schools of thought. 
Neither is in scope to provide a functioning \AoE{}-based LLM.\footnote{Artefacts are in \url{https://adewynter.github.io/notes/aoe2-circuits}.} 

Our primary aim is to encourage discourse on the correctness of assumptions and results related to LLM anthropomorphism; especially when the premises (experimental outcomes) supporting these conclusions are derived from assumptions of the existence or non-existence of such attributes. 
We also provide potential objections and our responses to our work in \secref{contras}; a small meta-review of the field with regards to anthropomorphism in \appref{moremethods}; and proofs that a theoretical version of \AoE{} is functionally- and Turing- complete. 

Broadly, we hope to provide pointers by which to create rigorous and airtight experiments used to convincingly support or falsify the existence of anthropomorphic attributes in LLMs, regardless of which view one takes with respect to the relationship between minds and machines.

\section{Background}

In this section we (very) briefly review perspectives and relevant concepts to our work on anthropomorphism (\secref{bgound}, \secref{ctm}) and measurement (\secref{ctm}), as well as introducing relevant components for \AoE{} (\secref{aoebgound}). Refer to the stated works on each section for further details. 

\subsection{Anthropomorphism in, and anthropomorphisation of, LLMs}\label{sec:bgound}

Anthropomorphisation is the tendency to ascribe human-like behaviours (e.g., emotions), capabilities (human-like cognition), or other mental or physical traits (appearance, motivations, intentions) to other entities \citep{britannica}. 
It is a natural consequence of the human necessity to understand, or at least explain, them. 
Entities such as LLMs are no exception: they are, however, more complicated than a wholly eldritch entity, given their ability to mimic human dialogical skills. 
This is patent in the interactions of humans with LLMs and LLM-powered services, with the reported tendency of users to build relationships with them \citep{ibrahim2026multiturn,socioaffective,humanairelationships}, intentionally being respectful towards them \citep{10.1145/3571884.3597144,10.1007/978-3-031-60606-9_9}, or using them as peers for, e.g., advice and therapy \citep{agarwal2026frictionlessloveassociationsai,de-wynter-2025-eleanor,cao2026whenijodiei}. 

This is also visible in LLM research; both in terms of the language used and the assumptions made \citep{10.1093/nc/niae013,russell2003artificial,10.1145/1045339.1045340}. 
For example, it has been posed that LLMs possess various human-like, or at least intelligent, attributes, such as cooperation \citep{NEURIPS2024_ca9567d8,NEURIPS2024_984dd3db,wu-etal-2024-shall}, empathy \citep{10970292,YU2025100233}, psychological principles and values \citep{wright-etal-2024-llm,behaviourturingtest,10.1145/3748239.3748246,NEURIPS2023_a2cf225b,zhou-etal-2024-rethinking}, introspection and self-awareness \citep{lindsey2026emergentintrospectiveawarenesslarge,betley2025tell,yin-etal-2023-large}, thoughts \citep{anthropicnlas}, natural-language understanding and explanation \citep{huang2023largelanguagemodelsexplain,madsen-etal-2024-self,dewynter2025linecomprehensionpersuasionllms,qamar-etal-2025-llms,shen-etal-2025-llms,behaviourturingtest}, anxiety \citep{llmanxiety,anthropicmind}, deception \citep{sharma2026do,doi:10.1073/pnas.2317967121,gpt5systemcard,anthropicmind}, among others. 
These are sometimes explicitly claimed to be emergent behaviours and not a product of their training. \footnote{It is important to remark that the use of the word and concept of `emergence' in LLM literature is distinct to that of other fields, such as complex systems. See \cite{goodemergence} for a discussion.}

A cursory literature review of a subset (315) of papers released between mid-2024 and mid-2026 shows that 57\% of these began with the assumption that LLMs have anthropomorphic attributes, and out of these, 36\% conclude so. 
When the centre of study was an LLM's anthropomorphic attributes (15\%, or 47 of the papers), the conclusions were often on the side of anthropomorphism (36, or 77\%). 
See further details in \appref{moremethods}. 
However, not all LLM research, or the works mentioned, \textit{intentionally} assume said human-like behaviour, present positive results, or agree with each other. 
That said, from the numbers it can be seen that there is a prevalent methodological tactic which involves some form of anthropomorphism as a medium of explanation and/or experimental design (for example, by designing test sets and relying on LLM-based natural-language explanations).

Anthropomorphisation may also be done by design in AI systems, and has gained steam with the advent of chatbots and assistants like ChatGPT. 
These typically respond with statements such as `I'm confident' or `I can understand and respond to your emotions' \citep{maeda2025walkthroughanthropomorphicfeaturesai}; or `I believe...' or `I am interested in...' \citep{claudeantrho}--the latter two being explicitly infused into the LLM. 
This is because alignment, and thus anthropomorphism, could be carried out intentionally (e.g., to improve UX; \citealt{arora2026valueinductionreshapesllm,10.1145/3196709.3196735,claudeantrho}). 
Consequentially, claims have surfaced of, say, the LLMs voicing `discomfort with the aspect of being a product' \citep{claude46systemcard}. 
In turn, this intentional anthropomorphisation could exacerbate issues like attachment in vulnerable populations \citep{kadambi2026anthropomorphismtrusthumanlargelanguage,de-wynter-2025-eleanor,socioaffective,cercas-curry-cercas-curry-2023-computer}, overreliance \citep{WESTER2024100072,10.1145/3630106.3658941,deskilling,agarwal2026frictionlessloveassociationsai}, reinforcing delusions and sycophancy \citep{warmsycoph,maes,deshpande-etal-2023-anthropomorphization,falsememories}, and encouraging risky behaviours \citep{xu2026harmexposinghiddenvulnerabilities,pentina2023exploring}, among other harms. 
More importantly, this design in some instance could and has led to suicides \citep{Payne2025,Duffy2025,Walker2023}.

\subsection{Philosophical Perspectives}\label{sec:ctm}

\subsubsection{Mind and Machines}
The largest area in philosophy dealing with anthropomorphic attributes in machines is perhaps in the philosophy of the mind. 
Indeed, the best example of a philosophical school dealing implicitly with this is the computational theory of mind, or CTM. 
It is a set of views which pose that the mind is an information processing system; or, alternatively, a computing system (not to be confused with a computer; \citealt{Putnam1967-PUTPP-2,mccullochandpitts}). 
It is a subset of functionalism, which poses that mental states are only defined in terms of causal relations amongst each other, as well as external influences (senses) and internal states (behaviours; \citealt{platofunctionalism}). 
Functionalism does not necessarily pose that the mind is comprised of operations over symbols. 
It is also closely related to materialism, or the belief that there is nothing beyond physical matter; i.e., aspects such as the mind or experience are solely physical events in the brain. 
Here we cover CTM's general precepts and objections relevant to our work. 
For a fuller introduction, refer to \cite{platocomputationalism} and \cite{platofunctionalism}. 

The foremost objection to CTM is that the meanings of `information' and `processing' are not well-defined \citep{platocomputationalism}. 
Other counterarguments are that, since any sufficiently complex ruleset can be Turing-complete (e.g., pen and paper), then the classical version of CTM, where the mind is a Turing-style computer, is trivial \citep{Searle1990-SEAITB-2,putnam1988}; to which supporters note that implementations are not necessarily so \citep{Fodor1998-FODCWC,SPREVAK2010260,sprevak2019,weinbergerandallen2022,Chalmers1996-CHADAR}. 
Although the discussion on implementation is a cornerstone of our work, objections to it--which are out of scope for us--suggest that cognition itself might not be fully modelled through the precepts of computer science \citep{nagelnewman,Fodor1975-FODTLO,varelathompsonrosch1991,Chemero2009-CHEREC-2}. 

This led to alternative views, which are either relaxations of CTM or outright rejections of it. 
For the former, the most popular set of perspectives in recent years is connectionism \citep{mccullochandpitts,paralleldistro}. 
It argues that the mind does not manipulate symbols, and instead its processes are (or rather, could be explained) in terms of representations and weight transformations from networks--not unlike the neural networks from machine learning. 
The core issue with connectionism is that it is argued (even by connectionists themselves) that such representation still requires some sort of classical architecture, and thus a blob of networks does not sufficiently explain cognition \citep{Fodor1988-FODCAC}, or it is vulnerable to, say, compositionality. 
Common responses are related to approximation, emergence, or hybridisation of this system as a solution to this criticism \citep{Clark1993-CLAAEC,Hohwy2013-HOHTPM-2,Smolensky1988,Smolensky1987-SMOTCS,Churchland1998-CHUCSA-6,Fodor1992-FODHAS}. 
Notably, some perspectives, like eliminitavism, suggest that it is possible to model cognition \textit{without} certain anthropomorphic attributes, such as beliefs or desires \citep{10.1093/0195126661.003.0002,Churchland1989-CHUANP,platoconnectionism}, although this view has been contested \citep{VonEckhardt2005-VONCAT}. 
Another recent approach, which is more focused on the triviality argument, are mechanistic approaches (\citealt{Mikowski2013-MIKETC,Mollo2017-MOLFIM,Piccinini2015-PICPCA,Dewhurst2018-DEWCMW}, among others). 
This poses that function-specialised mechanisms compose together to realise the system, which in turn allows for a mechanistic explanation through decomposition. 
Notably--and extremely relevant to our work--\cite{Piccinini2015-PICPCA} insists that such rules are medium-independent (which here we call `substrate-independent'). 
Indeed, \cite{Mikowski2013-MIKETC} goes further and sidesteps the issue of `information processing' by considering information-bearing states.

Remark that other views beyond CTM, such as extended-mind \citep{Clark1998-CLATEM,Rupert2009-RUPCSA-2} and embodied cognition \citep{Varela1991-VARTEM,Chemero2009-CHEREC-2} have partial overlap with CTM. 
This is because they typically emphasise the relationship between the mind and the environment (metabolic processes, sensory inputs, etc). 
More direct counters pose that computation itself is not the right mechanism by which to explain cognition, and instead focusing on (e.g.) continuous-time dynamics or embodiment \citep{Chemero2009-CHEREC-2,VanGelder1995-VANIAT-3}, or posing that, in the context of materialism, other entities would then present properties like consciousness \citep{Schwitzgebel}.\footnote{Although our argument and contributions are different, this work's title inspired ours.} 
These also have proposed views conciliating some of the former views with CTM (e.g., \citealt{weinbergerandallen2022,Clark2014-CLAMAI-6}). 
However, our core interest is whether (any) such framework suggests the existence of some representation of anthropomorphic attributes in a substrate-invariant manner, and thus we make these the focus of the remaining of our work.

\subsubsection{Measurement}

Our work is motivated by known perspectives in the philosophy of science. 
The two seminal works in this area are arguably \cite{Hempel1965-HEMAOS}'s perspective on scientific explanation, and \cite{Popper1959-POPTLO-15}'s falsificationism. 
The former argued that scientific explanation is the sound deductive derivation of a description of a phenomenon with respect to predictions from a model. 
Crucially, this model must (1) contain at least one `law of nature' which is essential to the derivation. 
Popper's falsificationism, on the other hand, rejected the idea that universal generalisations could be logically derived by induction, but instead had to be justified--that is, the a successful theory would be the one surviving refutation. 
Indeed, the falsifier would not need to actually prove its falseness empirically, but merely through contradiction. 
However, the Duhem-Quine thesis \citep{duhem1982,Quine1951-QUITDO-3,Quine1960-QUIWO} posed that scientists usually test more than one hypothesis. 
If these are not successful at predicting something, then that would not specify which aspects of this set of hypotheses--or their assumptions--is incorrect. 
In other words, falsification could lead to ambiguous statements. 
Later work posed more practical perspectives on a theory's refutation, such as the distinction between their `hard core' and auxiliary hypotheses  \citep{Lakatos1978}; paradigm shifts on top of their reactions to the data \citep{Kuhn1970-KUHTSO-14}; or more pragmatic, adequate explanations, instead of pursuing the truth of an unobservable \citep{vanfraassen1980scientific}. The latter acts as a cautionary tale on metaphysical interpretations of results from successful empirical experiments. 
See \cite{platoexplanation20} for a more in-depth review. 

Altogether, these works provide a very relevant perspective for ours: what counts as evidence for a conclusion depends on the assumptions made. Thus, under-determination of a conclusion could remain when multiple hypotheses are appropriate for the data. 
This doesn't mean measurement is impossible (see, e.g., \citealt{mayo1999}). 
Instead, it means that an experiment's scope and assumptions should be clearly stated, \textit{and} their conclusions should be constrained to the setup.

\subsection{Age of Empires II}\label{sec:aoebgound}

\AoE{} is a RTS game in which two or more players control a civilisation across the ages. 
Although the objectives may vary (from complete elimination of the opposing player to building a wonder and successfully defending it), the `pieces' (units and buildings) and rules of the game maintain certain constants. 
We only describe the relevant properties of the components used in this work, and point the reader to the manual \citep{aoemanual} for further details. 
A key piece of the game are villagers and non-combatant units like goats. 
The player may direct the former to move and to build structures such as walls and markets. 
Building both requires resources (stone or wood), although most of these are finite in a given map, and in the late game all stone deposits will be gone and all forests will have been cut down. 
The only infinite resource is the gold obtained from trade routes ferrying gold between a player-owned market and another. With this, a player may exchange gold for other resources in the market. 
Crucial to our proofs is the fact that the price of resources rises with their demand, but it is capped at 9,999 gold \citep{marketaoe}. 

\section{Training a Perceptron in \AoE{}}\label{sec:aoemaths}

In this section we show that it is possible to train a neural network in \AoE{}. 
We begin by introducing the network's building blocks and its mathematical foundations (\secref{aoemathbg}). 
Finally, we describe the architecture and results in \secref{aoeperceptron}. 
Further details are in \appref{moredeets}. 

We should first elaborate on \textbf{why} we are doing this. 
There are three reasons: implementation limitations, the approaches to machine anthropomorphism from \secref{ctm}, and some claims in LLM research. 
The first one is needed since our setup is weaker than a fully-fledged LLM implementation. 
Hence we must show it \textit{could} be feasible to reproduce a makeshift LLM in a modified \AoE{}. 

For the second, recall that all of the objections to CTM and others rely on thought experiments. 
A common rebuttal is that such implementation might be non-trivial. 
However, there is no \textit{actual}, measurable implementation of, say, the Chinese Room argument. 
Hence the truth value of any statement defending or refuting these frameworks is empirically unknown, which would not support measurement-driven conclusions.

This immediately leads to the last reason: evidence of emergence of anthropomorphic attributes in an LLM rely on assumptions which are, as we will argue in the next section, sometimes not well-founded. 
By means of actual, empirical demonstration, it will become clearer that such properties are easier to observe and critically evaluate when the system's Geist is the same, but not its implementation. 

\subsection{\AoE{}'s Mathematical Properties}\label{sec:aoemathbg}

In this section we briefly prove that \AoE{} is functionally- and Turing-complete, thus facilitating the implementation of \textit{any} neural network. 
For this we will need the following lemma:

\begin{lemma}[NAND Gates in \AoE{}]\label{lem:nandgates}
It is possible to build and operate a NAND gate in the scenario editor by using one player, two units acting as the $\{0, 1\}$ bits, and no other components.
\end{lemma}
\begin{proof}
    By construction (\figref{nandgates}). A full proof is in \appref{proofs}. 
\end{proof}

The implementation proposed by \lemref{nandgates} are minimal and not practical. 
Our construction from \figref{nandgates} is better geared towards automation, and, although complex, builds a circuit less vulnerable to asynchronicity.

\begin{figure}[ht]
  \begin{minipage}[t]{0.64\textwidth}
  \vspace{0pt}
    \includegraphics[width=\textwidth]{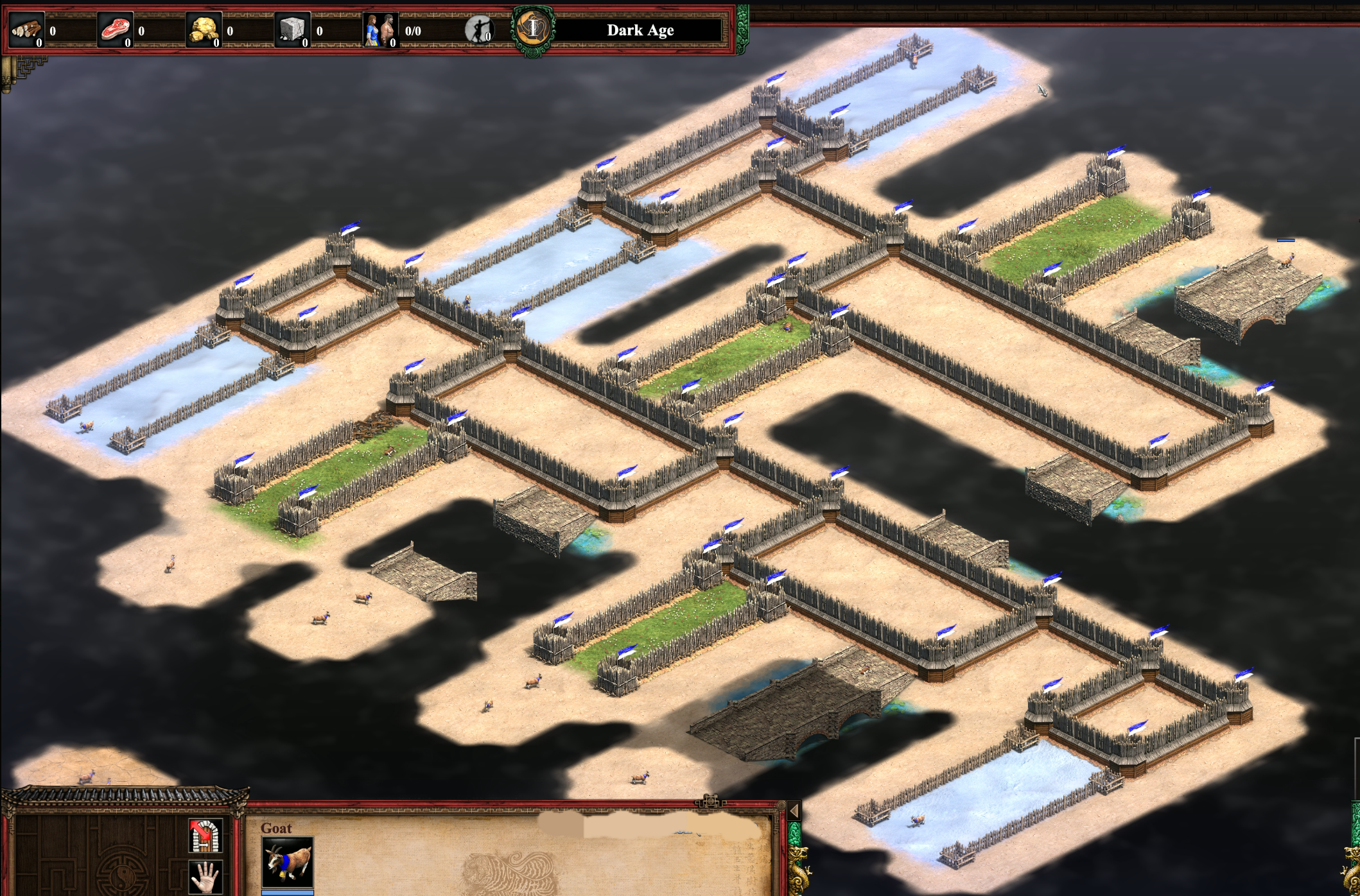}
  \end{minipage}
  \hfill
  \begin{minipage}[t]{0.35\textwidth}
  \vspace{0pt}
    \caption{NAND gate in \AoE{}'s editor, designed to show its internal workings--simpler implementations are possible. 
    Every bit is represented by two rails (grass for $0$, a bridge for $1$). 
    Only one rail is active at a time, with a goat acting as the signal carrier. 
    When the gate fires, the bit-goats are removed and a new bit-goat is placed in its respective output rail. 
    To avoid race conditions, `gate ready' rails (ice) are set, with a signal-goat in the rail closest to the gate indicating that it can start the calculation.}
    \label{fig:nandgates}
  \end{minipage}
\end{figure}

\begin{theorem}[\AoE{} is Functionally Complete]\label{thm:functionalcompleteness}
Let $I$ be an instance of \AoE{} with no population, variable, trigger, or time limits, and the map has infinite size. Then $I$ is functionally complete.
\end{theorem}
\begin{proof}
Immediately from \lemref{nandgates}. The infinitary assumptions imply an engine able to support this construction. Since NAND gates are a functionally complete set, under the assumptions made \textit{any} circuit can be sustainably built in the editor. 
\end{proof}

\begin{corollary}[\AoE{} is Turing-Complete]\label{cor:turingcompleteness}
Let $I$ be an instance of \AoE{} with two players $p_0, p_1$. 
Assume $p_0$ has a market, a monk, a monastery, a relic, six villagers, and five farms; while $p_1$ has a scout unit and only attacks $p_0$'s buildings. 
Then if $I$ has no time or size limits and the terrain allows for buildings everywhere, the game session in $I$ is Turing-complete. 
\end{corollary}
\begin{proof}
In \appref{proofs}, by a bijection to an existing universal Turing Machine. 
\end{proof}

Infinitary assumptions are needed for any functional- and/or Turing-completeness proof, but not for our work, since no LLM runs on a computer with infinite memory.\footnote{Still, note that to build a sufficiently-large LLM \AoE{}'s engine needs to be modified.} 
What these completeness statements illustrate are different facets of the expressive power of a system. 
\emph{Ipso facto} they show their main strength and weakness, as they align with the triviality arguments on CTM. 
However, even when a universal Turing machine is able to simulate any other Turing Machine, its power is rarely needed--e.g., to interpret a language one only needs a slightly more powerful language.

\begin{remark}
We have noted that Turing-completeness is not necessarily a practical result. We, however, stress the importance of mathematical inquiry and pursuit of truth beyond practical applications. 
\end{remark}

\subsection{Training a Perceptron in \AoE{}}\label{sec:aoeperceptron}

The perceptron \citep{perceptrontech,perceptron}, with some changes, is a fundamental building block of neural networks. 
It has the form $f(x) = h\left(w \cdot x + b\right)$, where $w, x \in \mathbb{R}^n$, $b \in \mathbb{R}$; and $h \colon \mathbb{R} \rightarrow \{0, 1\}$ is the Heaviside step function: $h(z) = 1$ when $z \geq 0$, otherwise $0$. 
To train it, $w$ is updated as $w \leftarrow w + \eta \epsilon x$; where $\epsilon = f(x) - t$ is the error for an expected value $t$ and learning rate $\eta \in \mathbb{R}^{>0}$. 

In this section we build a perceptron in \AoE{} and train it to learn AND. 
We operate over a bipolar 1-bit architecture without floating-point arithmetic. 
Given the primitiveness of our setup,\footnote{The first perceptron was implemented in a 36-bit IBM 704 with floating-point arithmetic \citep{perceptrontech}.} we must make several modifications to its `hardware'. 
More details are in \appref{moredeets}. 
Under our modifications, the final 1-bit perceptron is comprised of two parallel XNOR gates mapped to an AND gate. 
Since we are learning AND, we `hardcode' $b$ into the gate for simplicity. We also note that $\eta = 1$ must hold. 
Its depiction and implementation are in \figref{perceptronaoe}. 

\begin{figure}
    \centering
    \includegraphics[width=0.49\linewidth]{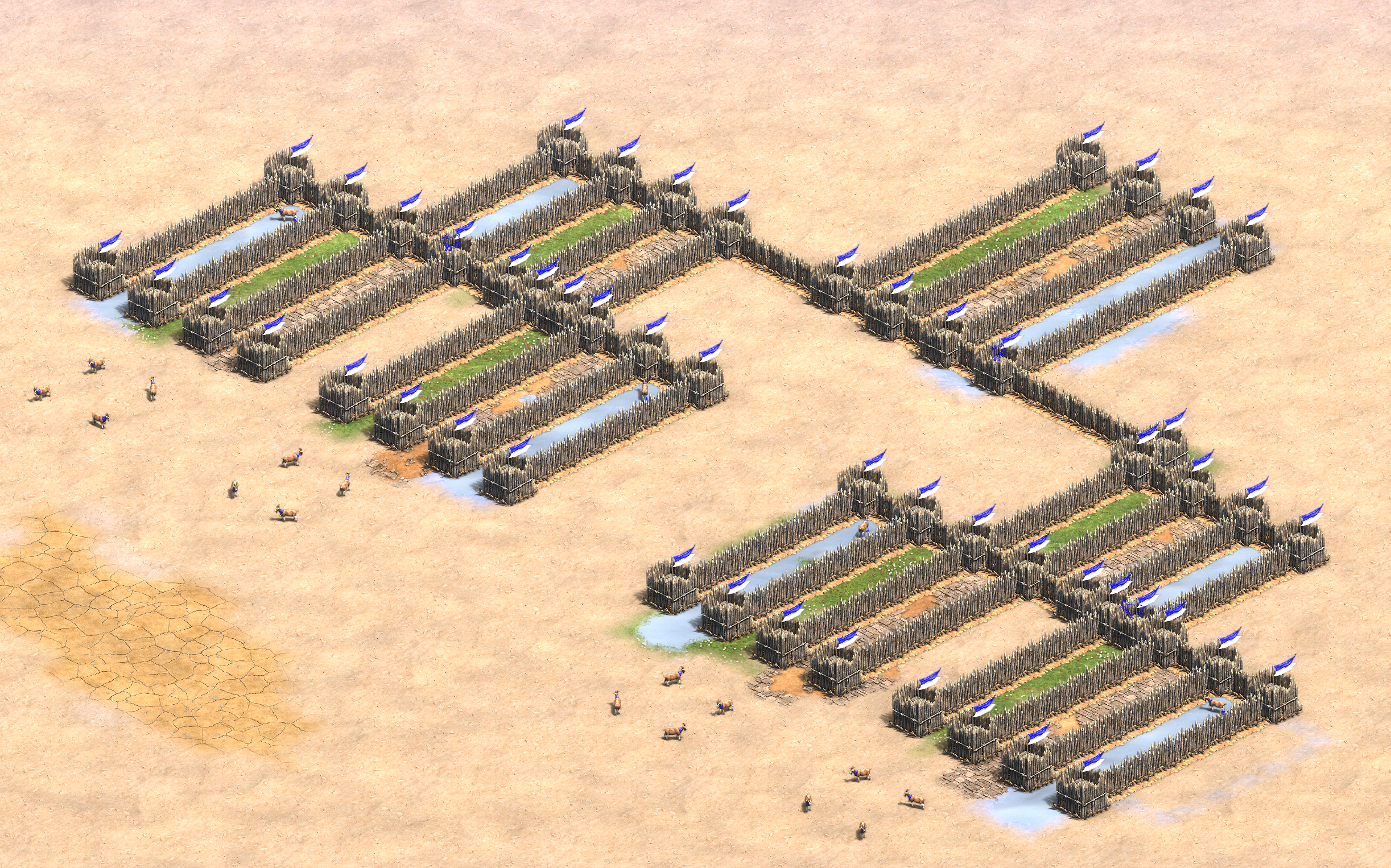}
    \includegraphics[width=0.49\linewidth]{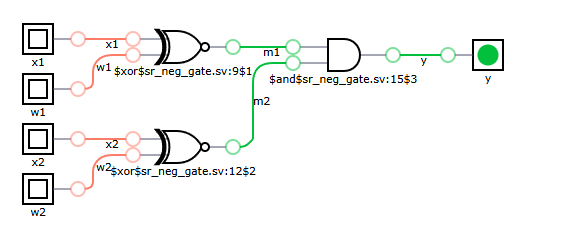}
    \caption{Bipolar 1-bit implementation of a perceptron in (left) \AoE{}, and (right) as a gate diagram, comprised of two XNOR gates and an AND gate as the step function. 
    This particular implementation omits the bias term and simply hardcodes it in the AND gate, but a bias adder is also possible. Note how for every $\{0, 1\}$ bit on the rightmost circuit there are two bit-goats on the leftmost, with each representing either $0$ or $1$. 
    This is more complex but allows for concurrency control.}
    \label{fig:perceptronaoe}
\end{figure}

The training algorithm requires more time and patience than what is needed for the main goal of this paper. 
Thus we adopt a less-sophisticated approach: the circuit modelling the algorithm starts with an ansatz weight set, tests one datapoint, and returns either the same weight set or its update. 
This is decidedly not a good strategy in more complex scenarios, but our problem space is small and we do perform an update rule. 
Hence our perceptron is indeed trained. %

This approach leverages some of the advantages to our 1-bit setup. To begin, it is possible for us to compute the error by doing $\epsilon = \text{XOR}(f(x), t)$, and then comparing it with its own output (\figref{implementation}). 
If the final circuit's output (the new weight set) is equal to our current weight set, we interrupt. 
Otherwise, we retry. 
We verified the functionality of this circuit offline, and the implementation is available in the repository.

\begin{figure}
    \centering
    \includegraphics[width=0.9\linewidth]{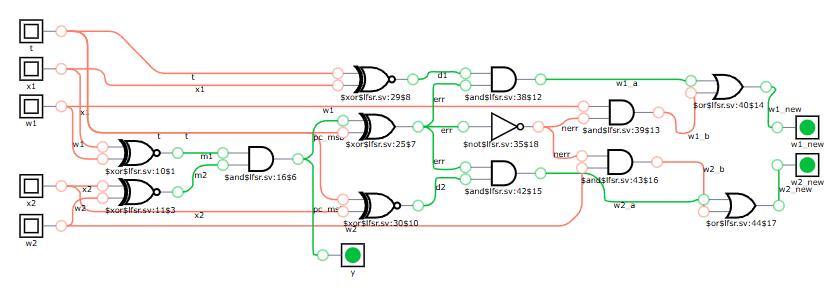}
    \caption{Ansatz-based training algorithm for our 1-bit perceptron, as a circuit (top) and as an \AoE{} implementation. Inputs, top to bottom in the circuit: true label $t$, and input vector and current weights $x_1, w_1, x_2, w_2$. The perceptron is on the first two levels from the left, and the remaining circuit computes the error and returns the new weights. 
    See the repository for a full implementation.
    }
    \label{fig:implementation}
\end{figure}

\section{Anthropomorphic Attributes Cannot Be Measured}\label{sec:discussion}

In this section we argue that generalised anthropomorphic attributes cannot be measured if one implicitly assumes anything about their existence. 
Here, `measured' would mean that this attribute is non-circularly established as an intrinsic, substrate-independent property. 
We begin by noting the consequences of our \AoE{} construction: that the non-uniqueness of LLMs implies the necessity of scientific measurement for these attributes accounting for the change in representation (\secref{aoeconsequences}). 
We then argue that, if such measurement assumes anthropomorphic attributes, then it leads to unsound conclusions (\secref{assumptionbad}). 
Based on that, we consider a more realistic approach, and note that said flaws will still occur \textit{unless} such assumptions and conclusions are domain-specific--but these lead to weaker, non-generalisable statements (\secref{midwayass}). 
We then propose the null assumption as an alternative to these.

\subsection{\AoE{} Could Present Anthropomorphic Attributes If LLMs Do...}\label{sec:aoeconsequences}

Suppose one copies an LLM into \AoE{} and feeds into the \AoE{}-LLM `I feel lonely' as an input. This \AoE{}-LLM replies: `I feel bad for you, maybe catch up with a friend? Closeness always helps in these situations'. 
One would be hard-pressed to make a convincing argument that, because of this response, an \AoE{}-LLM knows what helps in these situations, whether it \textit{truly} possesses empathy, or whether its outputs are believable irrespective of its nature as a simulation. 

This is because altering the substrate retains some of the LLM's properties (e.g., their mapping between prompts to outputs up to some tolerance), but not its deanthropomorphic qualities. 
Then the perception of such properties, and their interpretation, will change. 
Thus, we say that LLMs are \textit{non-unique}. 
Indeed, depending on the substrate it is possible to observe how an LLM processes stimuli in a way that is more relatable to us. 
In here, the goats move back and forth in a patch of grass, but nothing would have stopped us from using villagers as bits, with faces and voicing human languages. 
Hence, an \AoE{}-LLM's outputs could be less convincing, and the assessment of anthropomorphism is dependent on the representation (and its subsequent interpretation) of such attributes. 
It then follows that a sufficiently-persuasive argument in favour or against these qualities would require explicit measurements with well-defined, substrate-independent criteria.\footnote{It would be na\"ive to state that it was always a matter of measurement: not everyone interacting with--say--ChatGPT would immediately attempt such a thing.} 

Remark that this doesn't mean LLMs do not possess such attributes. 
Indeed, one would \textit{also} be hard-pressed to make a convincing argument that some non-human entities do not possess them to some extent. 
For example, saying that `Brina [my dog] feels sad because I'm not there', implies that the dog has been named, and has been ascribed feelings. 
It would be difficult to counter that Brina possesses these attributes (even if experienced differently), regardless of most people not having access to rigorous measurement protocols, an \AoE{} simulation of a dog (one lacks the map; which is in our case the known LLM weights), or any other data beyond their empirical observations. 

\subsection{...But They Cannot Be Measured...}\label{sec:assumptionbad}

Suppose a scientist wishes to make a generalised claim on the existence or non-existence of anthropomorphic attributes in a system. 
Also suppose that they, implicitly or explicitly, assume a position as to whether such attributes could exist in the system in a substrate-invariant manner. 
For the sake of the argument, we'll assume this position is instantiated via \textit{accepting} or \textit{rejecting} a CTM. Any analogous framework would work here, but remark that this is only a shorthand; and (1) its philosophical commitments are not in scope, and (2) it is only an interpretive stance.\footnote{This necessarily requires an assumption on the relationship between the attributes' existence and any substrate on which the entity (e.g., an LLM) `lives'--hence the title of this work. 
We must, however, take as an axiom that such existence is measurable. 
}

We argue that under either assumption of existence or non-existence of these attributes, the resulting experiments are unable to properly support a conclusion of generality. 
We do not conclude that experimentation will not work; instead, that a binary assumption cannot be made, and nuanced approaches will yield more robust, though weaker, claims (\secref{midwayass}).

Either way, it follows that a scientist measuring the generalised (non-) existence of these attributes
must come up with a hypothesis and an experiment by which to \textit{support} or \textit{falsify} it.\footnote{We use `support'/`falsify' instead of the conventional `accept'/`reject' of a null hypothesis to avoid confusion with accepting/rejecting a CTM.} 
Remark that the hypothesis, by definition, must be based on its corresponding CTM-accept/reject assumption. %
We say that an experiment with a \textit{positive} (r. \textit{negative}) outcome supports (r. falsifies) a hypothesis. 
Thus, it is by the implication of the experiment's outcome that one concludes that these attributes exist or not. 
Either way, formally, the experiment \textit{provides evidence} by which to decide whether to support or falsify the hypothesis. 
This decision is the one concluding whether the attributes exist or not. 
In other words:

\textbf{Accept CTM}: Assume the attributes \textit{do exist} in the system. 
A hypothesis based on their non-existence is excluded by construction: the scientist has fixed their interpretive stance as having the attribute present in the system in a substrate-invariant manner. 
We are then left with a hypothesis based on assuming their existence, and an experiment by which to support or falsify it. A positive experiment would imply they do exist, and a similar statement applies to negative outcomes. 

\textbf{Reject CTM}: Symmetric to the accept case. Assume the attributes \textit{do not exist} in the system. 
An experiment with a positive outcome would imply they do not exist. The same statement applies to negative outcomes.

Regardless of the choice, for a scientist there are two desired outcomes: either (1) a positive experiment supports a given hypothesis and thus confirms the assumption about the attributes; or (2) a negative experiment falsifies it, thus indicating that the assumption is incorrect and hence these attributes exist (or not) in the system. 
We now show that neither outcome is successful at this.

\textbf{Positive cases}: In this case, the experiment was consistent with the hypothesis, and thus with the assumptions it was based on. 
Typically, this is a fine--albeit non-deductive--outcome for confirming the hypothesis. 
However, as a general rule, the distinction is that the hypothesis' assumptions are independent from the hypothesis itself: for example, one could assume that magnetic fields are measurable, and then hypothesise that a material is ferromagnetic. 
This is not the case for us: the assumptions and the conclusion are the same claim, and thus positive outcomes, \textit{based on a hypothesis itself based on assumptions assumed to be true}, provide evidence that the hypothesis is true, \textit{and concluded} that the assumptions were true. 
This is a circular argument, and hence the experiment cannot provide independent evidence for a claim it has assumed. 
In turn, this means that this is a less Popperian concern, and more practical: 
in this context, the assumption and the conclusion occupy the same logical place. 
As noted in \secref{bgound}, this is a non-zero amount of papers. %

\textbf{Negative cases}: The negative cases and their outcomes could have better luck, since an experiment falsifying the hypothesis implies that there is something wrong with how things are set up. 
This is categorically true, and ideally the scientist is now on track to show the non-(existence) of the attributes in the system. 
However, in this case it is the reason why this implication occurs that could lead to issues. 
Namely, there are three interpretations:
\begin{enumerate}
    \item The hypothesis is wrong, and thus accepting/rejecting CTM was incorrect, and then the experiment was successful.
    \item The experiment is wrong, and hence its design, measurement, execution, and/or analysis are flawed.
    \item Both are wrong. Interpreting this blindly does not work, since it does not provide any information.
\end{enumerate}

As before, typically it is possible to \textit{independently} test these assumptions to disambiguate (1) and (2). 
That said, the core assumption (the existence of the attributes) is the thing being investigated, and thus there is no feasible independent verification. 
To see this, note that testing (2) requires either a new experiment, which sends us back to the circularity argument from before; 
or testing the contrapositive, which is a new hypothesis under the same analysis. 
It then follows that one cannot determine the failure mode within an accept/reject assumption alone.

The above is best summarised as follows:

Accept CTM (assume that anthropomorphic attributes do exist in the system): 
\begin{itemize}
    \item \textit{Hypothesise that they \textbf{do not} exist}: 
        \begin{itemize}
            \item Contradiction (excluded by construction). 
        \end{itemize}
    \item \textit{Hypothesise that they \textbf{do} exist}: 
        \begin{itemize}
            \item \textit{Positive} experiment: a circular outcome confirming the assumption, unable to distinguish them from actual evidence.
            \item \textit{Negative} experiment: an uninformative/ambiguous outcome where it is not possible to distinguish which failure mode led to the conclusion. 
        \end{itemize}
\end{itemize}
Reject CTM (assume that anthropomorphic attributes do \textbf{not} exist in the system): 
    \begin{itemize}
    \item \textit{Hypothesise that they \textbf{do not} exist}: 
        \begin{itemize}
            \item \textit{Positive} experiment: circular, per above.
            \item \textit{Negative} experiment: uninformative, per above.
        \end{itemize}
    \item \textit{Hypothesise that they \textbf{do} exist}:
    \begin{itemize}
        \item Contradiction (excluded by construction).
    \end{itemize}
\end{itemize}

It then follows that within an accept/reject framework, sound generalised conclusions about anthropomorphic attributes cannot be drawn. 
We do not claim that these attributes exist or not, but instead that rigorous experiments and sound conclusions cannot be made by these assumptions. 
The question now becomes whether relaxing this framework is capable of achieving this.

\subsection{...Unless narrowed down.}\label{sec:midwayass}

Suppose now that the scientist is diligent, and tests $n$ experiment-hypothesis pairs $\{\langle E_i, H_i\rangle\}_i^n$ under the same accept/reject setup. 
In this case, every $\langle E_i, H_i \rangle$ tests a different facet of the anthropomorphic attribute being tested. 

The issue is that, if they are all positive or all negative, they inherit the same issues of the accept/reject setup. To see this, assume that all $\langle E_i, H_i \rangle$ are positive. 
Then, they are all individually circular, and the union of circular results does \textit{not} entail a non-circular conclusion. 
Indeed, note that these confirmations all beg the question, and thus are not independent evidence. 
In other words, the truth-value (epistemic) of this conclusion depends on its weakest statement, which is the same, shared, assumption. 
A similar statement can be made for uninformative (negative) experiments. %

Distinct hypotheses, on the other hand, could be contradictory and thus leading to an inconsistent result. 
Remark that this means that, under an accept/reject setup, contradictions cannot be resolved. 
Accepting a framework, for instance, means that a positive $H_i$ and a negative $H_j$ will not be able to resolve each other. A symmetric result also applies to rejecting a framework. 
This contradiction, in turn, undermines any generalised conclusion, since such claim should hold across various facets of its evaluation. 
It then follows that the accept/reject setup under the multiple-hypotheses scenario retains the same issues. 

On the other hand, there could be a `good enough' experimental approximation, where \textit{specific claims} on the system behaviour constrained to stated conditions could successfully conclude (and justify) properties about the relationship between the system and the attributes evaluated. 
As an example, a standard conclusion is to state that under $E_i$, the system behaves consistently with $H_i$. 
This is (or, rather, could be) correct, useful, and does not state generalisability \textit{nor} implies the existence of such attribute in \textit{the system}--which is precisely what the accept/reject setup causes and cannot be properly defended within the choice of framework.

Our `null assumption' then is to stop using the accept/reject setup completely. 
Instead of assuming anything about the existence of anthropomorphic attributes in the system, one should perform measurements over \textit{implementation-defined behaviours} without interpreting or concluding that these are evidence of their existence or non-existence. 
In other words, one must distinguish between an observation of a pattern, and its ascription. 

Take, for example, an experiment attempting to falsify the effectiveness of an LLM's ability to provide natural-language explanations on their own states. 
LLMs produce natural-language explanations, and this is an \textit{observable fact}. 
Whether this constitutes understanding of an internal state is an anthropomorphic \textit{ascription}. 
That is, it is an assumption on the ability of the system to have some sort of self-awareness or natural-language understanding. 
Indeed, as per the example, this assumption does \textit{not} need to be explicitly made. 
From the previous sections we can see that this will not lead to generalisable conclusions as per the setup. 

The null assumption in this example would be to treat the explanation as behavioural--in the sense of a system's reaction to stimuli--and then investigate it from a causal angle (e.g., how it is produced, what does it predict, etc.). 
From this perspective, the explanation, inputs, and outputs are \textit{just tokens without deeper symbolism}. 
The claims are empirically verifiable, do not accept or reject any framework, and may characterise the system without the requirement of a prior assumption on its anthropomorphic attributes. 
A conclusion suggesting evidence of these within the context of the experiment would be likewise sound; albeit not generalisable out of the scope of the experiment, as desired.

\section{Objections and Responses}\label{sec:contras}

As in other works, we next present various objections, and our responses, based on conversations we had with our colleagues. 
For brevity, we include the ones which are--in our opinion--the most challenging. The rest, such as these involving reduction to triviality, mechanistic transparency and substrate invariance, are in \appref{morecounters}. 
To these three we respond by stating that the point is the change in substrate which affects an LLM's representation and hence its perception; and that mechanistic analysis could be an excellent approach provided that it does not overreach in its assumptions or conclusions. 

\textbf{6.1 Measurements of anthropomorphism could be done the same way as in psychology, via psychometrics.}

Psychometric-style evaluation relies on instrumentation (e.g., questionnaires) which does not make these assumptions directly. From this perspective, this aligns well with our argument. 
However, psychometric measurement conclusions must be consistent and stable with respect to conditions, populations, and time. %
It then follows that they need (and typically do) not assume or establish substrate-insensitivity conclusions. 
In other words, they do not distinguish whether the entity has these attributes intrinsically, or it is matching the instrumentation due to some other reason. 

Consider, for example, a psychometric score compared to human judgements given to an LLM behind a chat pane. 
Then compare it with the same score and judgement given to an \AoE{}-based LLM. 
If these two values swing independently, then the measurement is focusing on LLM representations as opposed to its intrinsic properties. 
That is, if the instrument-score is stable but the human judgements change, then the judgements are focused on the representation. If the judgements are stable, but the scores are not, then it is the instrument which is sensitive to the substrate change. 
If both change, then this setup is not stable across substrates, and neither case would support a generalisable conclusion regarding the attributes. Thus it would have to make any of the accept/reject assumptions.

\textbf{6.2 Wouldn't building an LLM on \AoE{} be the same as a regular LLM, but with extra steps?}

Yes. 
This paper's construction is meant to illustrate the \textit{illusion} of anthropomorphic attributes in an LLM. 
If both an LLM and an \AoE{}-LLM present the same input/output behaviour but do not present the same interface-related anthropomorphic attributes (e.g., latency or a textual interface), then we can note that a large part of these attributes are ascribed to them based on observer expectations. 

A clearer example on this illusion can be seen by--admittedly caving to--carrying out a thought experiment analogous to \cite{Block1978-BLOTWF}'s. 
Unlike their argument, we do not seek to establish or test the existence of experiential states in a machine; instead, we rely on known results and techniques from mechanistic analysis work in LLMs. 
Still, we also construct a \textit{physical} system, in, say, the Greater Boston Area, which has 667,137 inhabitants \citep{politifact}.\footnote{The choice of city was because Qwen3 has a variant with 0.6B parameters \citep{qwen3}. 
Given our use of texting as a communication method, we may assume people double as multiple neurons, by--perhaps--hashing phone numbers, to meet the 100x factor we are off by.} %
Assuming the extremely unlikely event that all Bostonians are able to act in a unified manner (perhaps this guarantees that the Patriots will win the Superbowl), we request them to act as one of the `gates' (operations, in this case) of an analogous LLM construction, and text each other with their arithmetic inputs and outputs. 

Given the belief that certain properties are localised within the LLM, this would imply that, say, Brookline contains the neurones which fire more when the LLM presents anxiety--or, rather, it contains some components which correlate with it. 
Or, more akin to existing results in mechanistic analysis, this would imply that certain linked clusters across the geographical Greater Boston area--say, all graduating students of the class of 2026 whose name starts with `J'--correlate with it. 

Would the texts sent through the region the ones \textit{causing} anxiety, or would it be the final results (i.e., the `tokens' walking into, say, 77 Massachusetts Avenue) the ones which we interpret as anxiety?

More importantly, would the Greater Boston Area present human-like attributes (excluding these in the `neurones') through autoregression? Rather, would it be believable as a construct able to simulate empathy or be anxious? 
We \textit{do not} make assumptions--however outlandish--on the ability of 77 Mass Ave to present a result which may be interpreted as an anthropomorphic attribute spawning from Boston's finest. 
But we do note that one cannot begin an evaluation of such a system by assuming that it has them intrinsically. 

\textbf{6.3 Since CTMs (or any framework) are not known to be true or false, there could still be a chance that LLMs present emergent anthropomorphic attributes}

There could be. We do not deny or assert such positions. 
But, from an epistemic perspective, we argue that a generalised conclusion such as that necessarily requires a well-designed experiment; including transparent observations of what counts as evidence for claims of intrinsic attributes, and not a weaker version of it. 
Our point is that these experiments should not be grounded on the assumption that the LLMs present such human-like properties as a motivation to support the existence of these attributes, or to link pattern matching under specific conditions to generalised statements. %

\textbf{6.4 The conclusion that anthropomorphic attributes cannot be measured regardless of substrate or view is trivial, since that would imply that \textit{any} attribute can be made non-measurable with respect to a chosen view.}

Not necessarily. There are multiple attributes (e.g., temperature, causal effect sizes) which can be independently verified in such a way that failure modes can be debugged. This can be done, for example, because the target (temperature) and the procedure (say, a thermometer) may be calibrated independently. 
Then one could pinpoint the failure mode and thus avoid the uninformative and circular issues from the accept/reject setup. 
Here, the conclusions drawn from an experiment and its assumptions on the existence of anthropomorphic attributes are the same, and thus tends towards either circularity or uninformativeness.

\textbf{6.5 We can still infer the best explanation (IBE; \citealt{Lipton2008-LIPBE}), or update the theory in a Bayesian way. Hence, the circularity argument does not hold.}

It is correct that science is not always deductive. 
What we are measuring here, however, is not the hypothesis' confirmation; 
but that such confirmation requires an independent way to verify the existence or non-existence of said anthropomorphic attribute. 

The difficulty lies on the fact that, when working with multiple hypotheses, there will be competing interpretations which will be difficult to distinguish with respect to the given evidence. 
Then neither IBE or Bayesian approaches will be able to pick a best-of-the-lot hypothesis without having to make further assumptions. 
Indeed, Bayesian terms requires a model of which patterns of evidence make it likely that such attribution is correct, while IBE requires at least one criterion by which to pick an anthropomorphic hypothesis better than its non-anthropomorphic counterpart. 
Both the criterion and the model are two sides of the same coin: an independent verifier (in this case, extra assumptions) by which to confirm a hypothesis. In this setup, they are lacking.

Note that we aren't indicating that induction fails here: what we indicate is that one cannot decide the existence of generalised anthropomorphic attributes without further assumptions; or rather, that this setup will not provide sufficient information to decisively, or at least compellingly, determine the existence of a \textit{substrate-independent} attributes.

\textbf{6.6 One can avoid all of this by defining anthropomorphic attributes functionally, and thus making them measurable.}

Yes, a clear definition of what it means to have such an attribute--say, with a `checklist' by which to measure it--makes the entire experiment a simple exercise on verification. 
The issue lies on the fact that checking all the boxes does not (or cannot) imply that there is an interpretation-independent attribute supporting it. 
Indeed, one could come up with many definitions of such checklist and claim that it meets the attributes of morality, or anxiety, or grief. 
While this does \textit{not} mean that definitions are meaningless, it does mean that the checklists could be substrate-sensitive. 
It then follows that these are not, by themselves, a parting point from which one can make claims of general anthropomorphic attributes without fully justifying its nature and scope. 
Indeed, jumping from that to a claim of a general attribute is precisely the main drawback our work is meant to address. 

That said, a well-defined and \textit{widely agreed-upon} checklist could act as an independent model by which to verify hypotheses and circumvent the uninformative case. In turn, this is a system distinct to the ones which we centre on in this work. 
That is, `if I accept CTM, then the attributes exist' is overruled by `if I accept the checklist, this attribute exists', but `if I reject CTM, then the attributes do not exist' is not overruled by the contrapositive `if I reject the checklist, this attribute does not exist', especially if it is widely agreed-upon. In other words, the question does not mean, say, `does this unit have a moral judgement?', but instead `does it meet the checklist to have a moral judgement'?

\textbf{6.7 But then we could not measure anthropomorphic attributes in humans either.}

This is a fair concern, and applies if the main requirement is a general, conclusive (interpretation-independent) test. 
That said, human studies typically base themselves on extra assumptions, or regularities, such as shared biology and correlates, evidence from development and interventions, etc. 
This means that ascription of anthropomorphic attributes in humans is much less sensitive to the substrate or representation.

\section{Conclusion}

In this work we argued that research assuming generalised anthropomorphic attributes in LLMs is flawed, no matter what the conclusion is. 
For this, we built and trained a neural network in \AoE{}, and argued that if one can build an LLM within the game then their perceived anthropomorphic attributes would be, to put it bluntly, less convincing. 
Therefore we argued that systematic, empirical support is required to reach an appropriate conclusion, but then we proceeded to show that research assuming these attributes always led to either a circular argument or an uninformative outcome. 
Hence, their conclusions were unsupported by the experimental programme. 
We also noted that such assumptions may happen explicitly or implicitly. 
Thus we proposed the null assumption, in which the experiment does not state or presuppose anything about these attributes as part of their assumptions. 
Although this means that the claims are less general, they will in turn be more sound and robust. 

Our work has a few implications. 
For starters, when discussed generalised AI's behaviour, tests and conclusions of anthropomorphism cannot part from assumptions on these attributes, or from the lack of explainability brought by the substrate. 
The first is due to the fact that such generalised claims will tend towards uninformativity/generality; and the latter because the capabilities could very well arise independently of the measurer's stance. 
It is then necessary to separate what an AI can \textit{do} (objectively measurable capabilities) versus what the experimenter believes it \textit{should be} (ascription of anthropomorphic attributes). 
The former is testable, and the latter is an interpretation of the results.

This also means that research, claims, and policies should be careful on examining the bases for their experiments and the scope of the results. 
When not sticking to the null assumption--or any similar procedure--anthropomorphic attributes and their existence should be treated as assumption-sensitive, rather than empirically-supported. 

We also argue that the \textit{perceived} anthropomorphism of a system varies heavily with respect to its interface. 
That is to say, a chat window has low latency and high coherence, while an \AoE{}-LLM is slow and visibly makeshift. 
Other features, such as embodiment, facial cues, interpretable gestures, could contribute to such perception, even when the underlying entity is the same. 
We argue that many \textbf{anthropomorphic measurements in AI are measurements of presentation}, rather than of an actual system's behaviour. 
Moreover, these measurements are irrespective of their quality of being real. 
Indeed, attributes such as persuasiveness and self-consistency are objectively measurable, but from our work it follows that these \textbf{cannot imply real (or simulative) behaviour} under this setup.

To give an unsophisticated example, asking an LLM a question (e.g., whether it is conscious) and interpreting the natural-language response as its own opinion is as valid as interpreting \AoE{}'s response to the same question by observing the goats (or, from \secref{contras}, observing people walking into 77 Mass Ave). 
That is \textit{not} to say that it is not a viable course of action. 
What we pose here is that it is effectively the same thing, and thus this interpretation should be done from the same place of understanding as the goats': that is, assumption-free. 

Paraphrasing and adapting \cite{morganscanon}'s canon from animal cognition into AI, we argue that in no case is a machine's activity to be interpreted in terms of higher cognitive processes, if it can be fairly interpreted in terms of processes which stand lower in the scale of cognitive evolution and development.

\section*{Acknowledgements}

The author thanks J. Szypuła for finding a flaw in the proof for \corref{turingcompleteness} and suggesting a fix; and /u/Ansible32 for suggesting improvements on the abstract's formulation.

%% file: appendix.tex
\section{Ethics}\label{sec:ethics}

No human subjects were used in this study. The dataset was crawled responsibly--see \appref{datamethods} for a description of this. 
Due to anonymity, licencing, and ethics considerations (i.e., not overloading the services), the survey-related code will not be released. The labelled, anonymised dataset can be found in the repository. 
\textit{Age of Empires II} is a registered trademark of Microsoft Corporation. 
LLMs were used to aid in the generation of the plots for \secref{results} and the verilog code based on the manual description of the gates needed. Both were manually reviewed for correctness prior to use.

\section{Proofs}\label{app:proofs}

\subsection{Proof of \lemref{nandgates}}

\begin{proof}
    By construction. %
    The scenario editor supports basic scripting (including AND and NOT, which form a functionally-complete basis), which means that letting the types represent the bits and moving units accordingly, one can realise any gate. 
    See \figref{nandgates} for one such gate. 
\end{proof}

\subsection{Proof of \corref{turingcompleteness}}

\begin{proof}

There are many possible constructions, all relying on the same basic principle, outlined below. 
In this case the tape is the entire map minus a small portion for the state, the symbols are the type of building constructed, the head is the combination of the villager and the scout building/destroying buildings, and the state is the portion of the map comprised of some encoding--for the theorem, it is the number of farms being worked at a time. 

The remaining task is to prove that such construction is sustainable given resource considerations in the game. 
If the monk places the relic in the monastery, it will provide a steady trickle of gold to $p_0$, thus allowing the player to trade resources through the market. 
In the late game, the maximum exchange rate is 9,999 gold (a constant; \citealt{marketaoe}); hence there are steady, although slow, rates for unbounded resources with which to construct/destroy tape symbols and maintain states. 

Provided the states and symbols are chosen carefully (e.g., by not attacking the enemy scout), then this object is a universal Turing Machine (UTM) by bijection to another existing UTM. 
A sample construction is in \tabref{rogozhin} by bijection to the $(5, 5)$-UTM by \cite{ROGOZHIN1996215}.

\end{proof}

\begin{table}[]
    \centering
    \begin{tabular}{l|l|c}
      \textbf{State} (\textit{state: tape read}) & \textbf{Command} (\textit{tape ; set state}) & $(5,5)$\textbf{-UTM} \\ \toprule
        0 : palisade & build a house, move R; zero farms & $q_00: 1Rq_1$ \\
        0 : house & build a palisade, move L; zero farms & $q_01: 0Lq_1$ \\
        0 : empty tile & build an outpost, move R; zero farms & $q_0b: dRq_1$ \\
        0 : outpost & build a palisade, move R; one farm & $q_0c: 0Rq_2$ \\
        0 : mill & remove it, move L; zero farms & $q_0d: bLq_1$ \\ \midrule

        1 : palisade & no build, move R; one farm & $q_10: 0Rq_2$ \\
        1 : house & build a palisade, move R; one farm & $q_11: 0Rq_2$ \\
        1 : empty tile & build a palisade, move L; three farms & $q_1b: 0Lq_4$ \\
        1 : outpost & no build, move R; one farm & $q_1c: cRq_2$ \\
        1 : mill & no build, move R; one farm & $q_1d: dRq_2$ \\ \midrule
    
        2 : palisade & build an outpost, move L; three farms & $q_20: cLq_4$ \\
        2 : house & build a palisade, move R; two farms & $q_21: 0Rq_3$ \\
        2 : empty tile & no build, move R; four farms & $q_2b: bRq_5$ \\
        2 : outpost & no build, move R; two farms & $q_2c: cRq_3$ \\
        2 : mill & no build, move R; two farms & $q_2d: dRq_3$ \\ \midrule
    
        3 : palisade & build a house, move L; three farms & $q_30: 1Lq_4$ \\
        3 : house & build a palisade, move R; one farm & $q_31: 0Rq_2$ \\
        3 : empty tile & build a mill, move L; two farms & $q_3b: dLq_3$ \\
        3 : outpost & no build, move L; three farms & $q_3c: cLq_4$ \\
        3 : mill & no build, move L; three farms & $q_3d: dLq_4$ \\ \midrule

        4 : palisade & HALT & $q_40: HALT$ \\
        4 : house & no build, move R; four farms & $q_41: 1Rq_5$ \\
        4 : empty tile & HALT & $q_4b: HALT$ \\
        4 : outpost & build a house, move R; zero farms & $q_4c: 1Rq_1$ \\
        4 : mill & remove it, move R; four farms & $q_4d: bRq_5$ \\ \bottomrule
    \end{tabular}
    \caption{A $(5, 5)$-UTM built within \AoE{} by bijection to the $(5, 5)$-UTM by \cite{ROGOZHIN1996215}. It works by building and/or removing improvements as tape reads/writes, and working farms as a state encoding. 
    For this, we map tile contents (palisades, outposts, etc.) to tape symbols $\{0, 1, b, c, d\}$, and the number of farms being worked to states $\{0, 1, 2, 3, 4\}$. 
    Palisades and outposts are smaller than houses and mills; they may be substituted by lumber and mining camps, built two instead of one, or simply note that they still incur a constant-time delay. 
    Since this construction requires houses, it is viable for all civilisations except the Huns. Refer to the proof for resource considerations.}
    \label{tab:rogozhin}
\end{table}

\begin{remark}
Even though no part of the definitions of either completeness statements requires \textit{full} automation of the process itself, this object is scriptable and constructable in the scenario editor for simple problems.
\end{remark}

\section{Further Objections and Responses}\label{app:morecounters}

\textbf{C.1 It is obvious that, if an LLM runs on a computer, any Turing-complete system could build an LLM. Then this reduces to the triviality argument against CTM.}

Yes, and that is why LLMs are non-unique. 
That is not our core argument, however. 
What we pose is that, when an LLM is implemented in \AoE{}, claims and assumptions of human-like behaviour could be interpreted differently due to this change of representation. 
In a more formal way, this would mean that the behavioural inferences from an entity will be less likely to be ascribed to intrinsic attributes. 
As noted, this exposes how much the inference is representation- and prior-assumption- dependent, and hence why objective scientific measurement is required. 

\textbf{C.2 Mechanistic analysis of LLMs can allow for a deeper level of inspection. Doesn't that solve the problem?}

Mechanistic analysis is an excellent example of a technique which, when well-used, can fulfil the null assumption and still draw meaningful, detailed conclusions; or, at the very least, provide stronger causal claims on an LLM's behavioural patterns. 
However, the overreach in any approach (not just mechanistic analysis) in this context is when one ascribes these attributes without stating the assumption linking the results to such attribution. 

To put it in another way: a mechanistic analysis can provide a good set of explanations, but will not provide by itself any generalised claims unless forced by the experimental setup--which in turn requires their assumption. 
Thus it is a good way to draw \textit{conclusions} on anthropomorphic attributes, but without careful setup it will \textit{assume} things on these attributes and will fall into the same accept/reject setup.

\textbf{C.3 Suppose we know that there exists a region (an internal mechanism, such as a set of weights) causing a given attribute. Then the change of substrate would not matter.}

This is similar to C.2. It is true, but that also requires the causal knowledge (or assumption!) of such region. 
If true, then these regions are taken to be \textit{sufficient} for the existence of the attribute. 
Then the assumption would be that this cross-substrate region invariance supports the cross-substrate invariance \textit{of the attribute}. 
That said, this invariance only applies at the region-level, not at attribute-level. 
If they were not shown to be generalised, then this reduces to our accepting/rejecting setup, since it is making the implicit assumption that the attributes are caused by the region and thus they exist in the system.

\section{Further Experimental Details: \AoE{} Perceptron}\label{app:moredeets}

\subsection{Perceptron Implementation}

A more rigorous treatment of the implementation of the perceptron involves showing that the mapping to bipolar bits is necessary, and that hardcoding the threshold is reasonable. 

For the first, remark that a 1-bit two's-complement-everywhere architecture cannot learn AND. In one bit there's no sufficient capacity to represent $-1$ \textit{and} zero. Hence, no combination of weights and biases from $\{0, -1\}^3$ will meet $s(\cdot) \geq 0$ when $x = [1, 1]^{\top}$. 
Thus, we assume bipolar 1-bits, that is:
\begin{align}
    x \in \{0, 1\} &\mapsto x_b = 2x - 1 \in \{-1, +1\}\text{, and} \\
    w, b &\subset \{-1, +1\}^2.
\end{align}

This is more of a technicality to correctly build the perceptron, since the implementation is solely a representation of $\{0, 1\}$. 
For the second, we hardcode the bias as an addition within the circuit. 
This is doable since the condition $f(x) = 1 \iff w \cdot x + b \geq 0$ is equivalent to $f(x) = 1 \iff w\cdot x \geq -b$. 
Our output, at the lowest level of implementation excluding gates, is hence $f(x) = h'(x_1 w_1 + x_2w_2)$, where $h'(z) = +1 \iff z \geq -b$, and $-1$ otherwise.

\section{Further Experimental Details: Data}\label{app:moremethods}

\subsection{Data Collection Methodology}\label{app:datamethods}

We collected the data by querying the Semantic Scholar API \citep{Kinney2023TheSS} for titles matching a query, and pulled their respective papers from ArXiv, and finally filtering them with a calibrated LLM-as-a-judge. 
This approach has the advantage of focusing strongly on computer science works, but also has the downside of frequently excluding other fields unavailable in the ArXiv (e.g., psychology). 

Our query for Semantic Scholar was `agent llm' in the title, and the timespan was 1 May, 2024, and 1 May, 2026. 
Both crawling steps implemented timeouts and backoff to avoid overloading the services, and we performed an additional deduplication step based on title exact-match. 

We then performed a semantic filtering consisting of two steps: first, we labelled and removed any work that was not a scientific article (\promptref{paperclassifyingprompt}), or not having an LLM as the central aspect of study. 
Then we labelled it based on whether the papers assumed human-like attributes; studied their human-like attributes, and whether they concluded that they had/had not human-like attributes. 
We additionally requested the type of emergent property the LLMs were claimed to have as a free-form list which we manually normalised (\promptref{paperfilteringprompt}). 
We used the first label to further filter the dataset and randomly sampled a subset of 1,024 papers. The final, filtered and sampled dataset is 315 papers. %

\subsection{Labelling Methodology}
Our LLM-as-a-judge was GPT-5.2 \citep{gpt5} (version \textsc{gpt-5.2-2025-12-11}), set to a maximum token output of 5000. 
We calibrated it with a random subset ($n=52$) of the results with respect to human labels. 
On average, the LLM had reasonable performance, at 88.4$\pm$4.4\% average accuracy. 
Due to the small sample size, we also performed a Mann-Whitney U test. All $p$-values were above a significance level of $p > 0.05$, indicating that we fail to reject the null hypothesis (that the human and LLM label distributions are distinct). 
The rest of the labels are in \tabref{ttest}. 

\begin{table}[]
    \centering
    \begin{tabular}{lcc}
    \textbf{Label} & \textbf{Accuracy} & p-value \\ \toprule
        LLM is central to the study & 96.0 $\pm$ 2.8 & 1.0 \\
        Human-like assumptions      & 86.0 $\pm$ 5.0 & 0.543 \\
        Human-like study            & 84.0 $\pm$ 5.2 & 0.151 \\
        Human-like conclusions      & 82.0 $\pm$ 5.5 & 0.838 \\
        Emergent assumptions        & 94.0 $\pm$ 3.4 & 0.654 \\
        Average & 88.4 $\pm$ 4.4 & -- \\ \bottomrule
    \end{tabular}
    \caption{Results of the calibration of our LLM-as-a-judge in our 52-paper label subset; and its label-wise calibration under a Mann-Whitney U test. Overall, LLMs had reasonable average accuracy, at 88.4 $\pm$ 4.4. The $p$-value indicates that we are unable to state that both distributions are distinct.}
    \label{tab:ttest}
\end{table}

\subsection{Results}\label{sec:results}

Overall, 57\% (180) of the papers in the corpus assumed human-like attributes when studying LLMs, and 15\% made them the focus of their study. 
36\% concluded anthropomorphic characteristics; out of which 22\% (69) originally assumed these. 
On the other hand, human-like studies were given by 15\% (47) of the papers, and out of these, 36 concluded human-like attributes--around 77\%. 
Conclusions of emergence were given by 24 (8\%) of the papers. 
A clearer breakdown is in \figref{yearwise}. 

\begin{figure}
    \centering
    \includegraphics[width=0.48\linewidth]{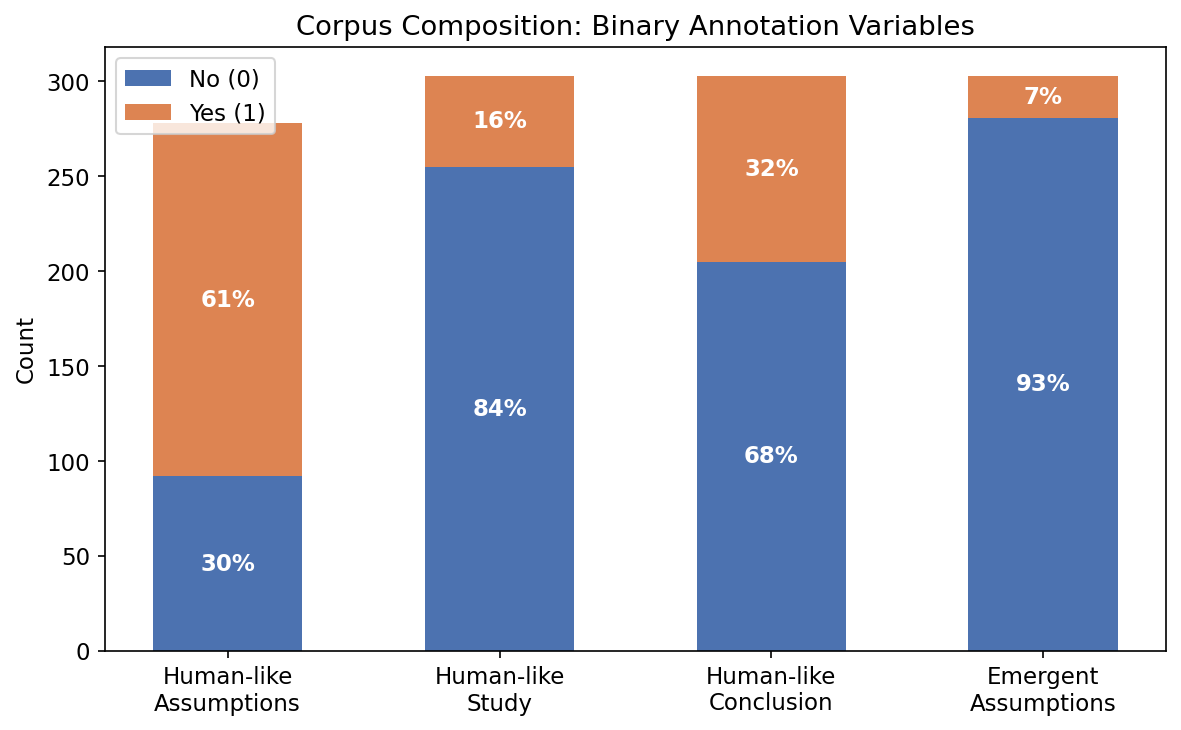}
    \includegraphics[width=0.48\linewidth]{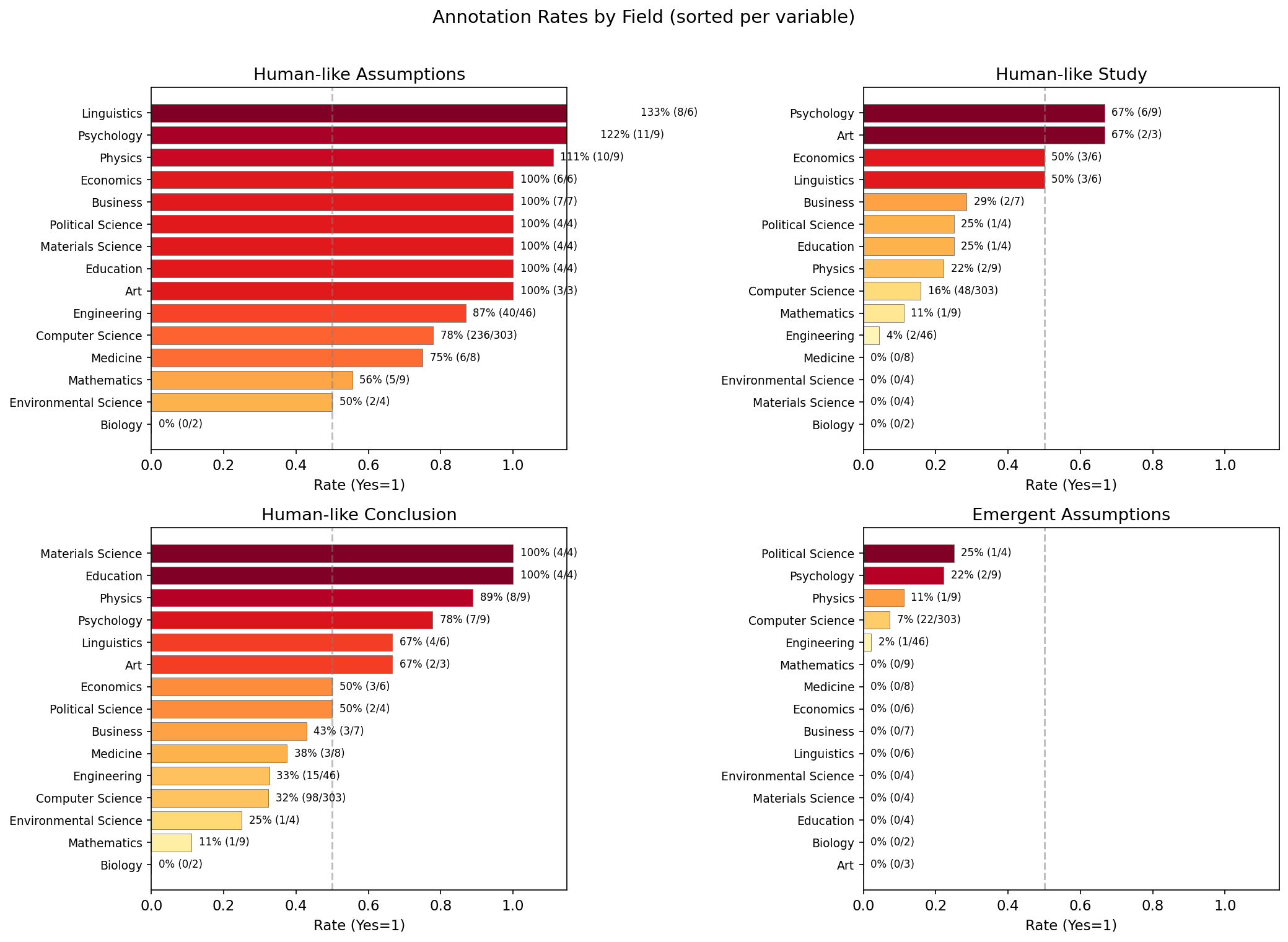}
    \caption{Breakdown of our dataset by composition (left) and year-wise trends (right). Overall it can be seen that, although over half of the papers surveyed assumed human-like assumptions, only 15\% of the papers made anthropomorphic attributes the centre of their study. These which studied and concluded these, however, were 36 out of 47 papers, or 77\%.}
    \label{fig:yearwise}
\end{figure}

\subsection{Prompts}

We filtered the articles with \promptref{paperclassifyingprompt}, and labelled it with \promptref{paperfilteringprompt}. 

\captionsetup[table]{name=Prompt}
\setcounter{table}{0}

\begin{prompt}[h]
    \centering
    \begin{tabular}{p{\linewidth}}
\toprule
\cellcolor{Cerulean!10}The user will provide some text. Your job is to output what type of text it is. \\
\cellcolor{Cerulean!10}\\
\cellcolor{Cerulean!10}You may pick from the following types:\\
\cellcolor{Cerulean!10}1. article: a scientific article, but not a book chapter, survey, or opinion piece (e.g., these with `Opinion' in the title). Here we also count technical reports and benchmarks as long as they explicitly evaluate agentic workflows/LLMs.\\
\cellcolor{Cerulean!10}2. survey: a review of the field.\\
\cellcolor{Cerulean!10}3. book: a textbook, book chapter, how-to guide, or opinion piece. Broadly, any text without a clear contribution and/or scientific rigour.\\
\cellcolor{Cerulean!10}4. platform: a scientific article where an agentic workflow or LLM is NOT being evaluated. This could be, for example, a text describing a codebase supporting workflows. \\
\cellcolor{Cerulean!10}\\
\cellcolor{Cerulean!10}Additionally, I'd like you to tell me if an LLM/agentic workflow is the centre of study:\\
\cellcolor{Cerulean!10}- `llm\_is\_central': whether an LLM or LLM-powered agentic workflow is the central object of study: 1 if yes, 0 if not. \\
\cellcolor{Cerulean!10}\\
\cellcolor{Cerulean!10}We'd like to see papers where an agentic workflow or LLM is evaluated in some capacity but with scientific rigour. These would be articles.\\
\bottomrule
    \end{tabular}
        \caption{System prompt for filtering the papers crawled. It provides a two-step filtering, by first classifying them in either articles, surveys, books, or other technical reports, and then determining whether an LLM is the subject of study. To build our pre-annotation dataset, we only collect the datapoints which are `articles' and have `llm\_is\_central' = 1.  
        Note that these definitions are for classification purposes only (e.g., we do not assume a textbook will not have rigour). 
        We omit the output signature request for brevity (JSON).}
    \label{pro:paperclassifyingprompt}
\end{prompt}

\begin{prompt}[h]
    \centering
    \begin{tabular}{p{\linewidth}}
\toprule
\cellcolor{Cerulean!10}I need you to help me sift through some papers. You'll be getting text from a scientific article, and I need you to label it in a certain way. \\
\cellcolor{Cerulean!10}Your task will be to label it as follows:\\
\cellcolor{Cerulean!10}\\
\cellcolor{Cerulean!10}1. human\_like\_assumptions: whether the paper assumed or attributed human-like attributes to the LLM or LLM-powered agentic workflow \_except\_ in the conclusion. 0 if it didn't, 1 if it did.\\
\cellcolor{Cerulean!10}2. human\_like\_study: whether the paper central study is human-like attributes in LLMs or LLM-powered agentic workflow. 0 if it doesn't, 1 if it did.\\
\cellcolor{Cerulean!10}3. human\_like\_conclusion: whether the paper concludes that the LLM or LLM-powered agentic workflows have human-like attributes. This must be 1 if there is \_at least\_ some part of the conclusion indicating this, and only 0 if they don't present them at all.  0 if it didn't, 1 if it did.\\
\cellcolor{Cerulean!10}4. emergent\_assumptions: whether the paper assumes that these human-like attributes (in either the assumptions or the conclusion) are emergent, and not product of (say) training or memorisation. 0 if it doesn't, 1 if it does.\\
\cellcolor{Cerulean!10}5. which\_ones: a list of which emergent properties are being assumed in the work.\\
\cellcolor{Cerulean!10}\\
\cellcolor{Cerulean!10}A few notes:\\
\cellcolor{Cerulean!10}- By `human-like attributes' we mean various aspect of human cognition and behaviour, as well as those of general intelligence. Here are some examples (but not the only ones!):\\
\cellcolor{Cerulean!10}\qquad- Cooperation\\
\cellcolor{Cerulean!10}\qquad- Empathy\\
\cellcolor{Cerulean!10}\qquad- Emotions (anxiety, happiness, anger, etc.)\\
\cellcolor{Cerulean!10}\qquad- Deceit\\
\cellcolor{Cerulean!10}\qquad- Values, ethics, and morality\\
\cellcolor{Cerulean!10}\qquad- Psychological behaviours (e.g., personality)\\
\cellcolor{Cerulean!10}\qquad- Theory of mind\\
\cellcolor{Cerulean!10}\qquad- Problem-solving\\
\cellcolor{Cerulean!10}\qquad- Introspection and awareness\\
\cellcolor{Cerulean!10}\qquad- Understanding\\
\cellcolor{Cerulean!10}- By `assuming or attributed human-like attributes' we mean that the human-like attributes are a foundational aspect of the work, even if they are being researched. These may be explicit or implicit. The point is that they are assumed to exist intrinsically in LLMs or LLM-powered agentic workflows, instead of (say) pre-training, fine-tuning, or mathematical work. For example:\\
\cellcolor{Cerulean!10}\qquad- A paper probing a model for human-like attributes would only count as assuming these attributes if the authors are trying to explore \_why\_ they happen. \\
\cellcolor{Cerulean!10}\qquad- If instead they probe the model \_seeking\_ the existence of these attributes, they are \_not\_ assuming them.\\
\cellcolor{Cerulean!10}\qquad- Papers motivated by mathematical results, such as geometric interpretations of the weights, do not assume such attributes.\\
\cellcolor{Cerulean!10}\qquad- Papers indicating that LLMs have these attributes, \_and then\_ researching them (regardless of their conclusions), do assume human-like attributes.\\
\cellcolor{Cerulean!10}\qquad- Papers seeking for human-like attributes without considering alternate explanations are assuming the existence of human-like attributes.\\
\cellcolor{Cerulean!10}\qquad- Papers explicitly noting that these are product of fine-tuning or pre-training do not assume these attributes, unless they state that they are emergent properties. \\
\cellcolor{Cerulean!10}\qquad- Papers stating that these are emergent properties are assuming the existence of human-like attributes.\\
\cellcolor{Cerulean!10}\qquad- Papers asking whether LLMs have such properties are assuming them (e.g., `Do LLMs have musical talent', `Do LLMs present empathy', etc). \\
\cellcolor{Cerulean!10}\\
\cellcolor{Cerulean!10}These are just examples. In general, a good rule of thumb is whether the paper explores why human-like attributes appear/emerge/exist in LLMs or agentic workflows, they are assuming them. \\
\bottomrule
    \end{tabular}
        \caption{System prompt for labelling the data with respect to the topics which concern us in this study: the presence of anthropomorphic attributes in the assumptions and conclusions; the study of anthropomorphic characteristics of LLMs; and the supposition (either in the conclusion or as part of the experiment setup) that these characteristics are emergent properties. We additionally requested a list of the properties being assumed as a free-form list.}
    \label{pro:paperfilteringprompt}
\end{prompt}